\documentclass[dvipsnames,12pt]{article}
\usepackage[utf8]{inputenc}
\usepackage{microtype,amsmath}
\usepackage{graphicx}
\usepackage{subfigure}
\usepackage{booktabs}
\usepackage{hyperref,amsthm}
\usepackage{amsfonts}

\usepackage{color}
\newcommand{\R}{\mathbb{R}}
\newcommand{\E}{\mathbb{E}}

\usepackage{caption}
\usepackage{algorithm,algorithmicx}
\usepackage[english]{babel}
\usepackage{algpseudocode}
\newtheorem{theorem}{Theorem}
\newtheorem{question}{Question}
\newtheorem{lemma}{Lemma}
\newtheorem{online_problem}{Online Learning Problem}
\newtheorem{definition}{Definition}
\newtheorem{problem}{Problem}
\newtheorem{claim}{Claim}
\newtheorem{remark}{Remark}
\newtheorem{example}{Example}
\usepackage[margin=1in]{geometry}
\title{Efficient Online Learning for Dynamic k-Clustering}

\makeatletter
\def\blfootnote{\xdef\@thefnmark{}\@footnotetext}
\makeatother

\author{%
  Dimitris Fotakis\\
  National Technical University\\ of Athens\\
  fotakis@cs.ntua.gr\\
  \and 
  Georgios Piliouras\\  Singapore University of Technology\\ and Design\\
  george.piliouras@gmail.com\\
  \and  
  Stratis Skoulakis\\
  Singapore University of Technology\\ and Design\\
  efstratios@sutd.edu.sg
}
\date{}

\begin{document}
\maketitle
\begin{abstract}
We study dynamic clustering problems from the perspective of online learning. 
We consider an online learning problem, called \textit{Dynamic $k$-Clustering}, in which $k$ centers are maintained in a metric space over time (centers may change positions) such as a dynamically changing
set of $r$ clients is served in the best possible way. The connection cost at round 
$t$ is given by the \textit{$p$-norm} of the vector consisting of the distance of each client to its closest center at round $t$, for some $p\geq 1$ or $p = \infty$. We present a \textit{$\Theta\left( \min(k,r) \right)$-regret} polynomial-time online learning algorithm and show that, under some well-established computational complexity conjectures, \textit{constant-regret} cannot be achieved in polynomial-time. In addition to the efficient solution of Dynamic $k$-Clustering, our work contributes to the long line of research on combinatorial online learning.
\end{abstract}

\section{Introduction}
\label{s:intro}

\textit{Clustering problems} are widely studied in Combinatorial Optimization literature due to their vast applications in Operational Research,
Machine Learning, Data Science and Engineering
\cite{WS11,LIN,CGTS99,VGRMMV01,CG99,JV01,K12,Y00,LS16,CL12,KSS10,SS18}. Typically a fixed number of centers must be placed in a metric space such that a set of clients is served the best possible way. The quality of a clustering solution is captured through the \textit{$p$-norm} of the vector consisting of the distance of each client to its closest center, for some $p\geq 1$ or $p = \infty$. For example \textit{$k$-median} and \textit{$k$-means} assume $p=1$ and $2$ respectively, while \textit{$k$-center} assumes $p=\infty$ \cite{LIN,KSS10,SS18}.

Today's access on vast data (that may be frequently updated over time) has motivated
the study of  
clustering problems in case of \textit{time-evolving clients}, which dynamically change positions over time \cite{KW18,FKKLSZ19,EMS14,ANS17}. In time-evolving clustering problems, centers may also change position over time so as to better capture the clients' trajectories. For example, a city may want to reallocate the units performing rapid tests for Covid-19 so as to better serve neighborhoods with more cases, 
the distribution of which may substantially change from day to day. Other interesting applications
of dynamic clustering include viral marketing, epidemiology, facility location (e.g. schools, hospitals), conference planning etc. \cite{JV11,EMS14,N3,PS01,CBK07}.

Our work is motivated by the fact that in most settings of interest, clients can move in fairly complicated and unpredictable ways, and thus, an \textit{a-priori knowledge} on such trajectories is heavily under question (most of the previous work assumes perfect knowledge on clients' positions over time \cite{EMS14,ANS17,KW18,FKKLSZ19}). To capture this lack of information we cast clustering problems under the perspective of \textit{online learning} \cite{H16}. We study an online learning problem called \textit{Dynamic $k$-Clustering} in which a \textit{learner} selects at each round $t$, the positions of $k$ centers trying to minimize the connection cost of some clients, the positions of which are unknown to the learner prior to the selection of the centers.    

\begin{online_problem}[Dynamic $k$-Clustering] Given a metric space $d:V \times V \mapsto \R_{\geq 0}$. At each round $t$,
\begin{enumerate}
    \item The learner selects a set $F_t \subseteq V$, with $|F_t| = k$, at which centers are placed.
    \item The adversary selects the positions of the clients, denoted as $R_t$ (after the selection of the positions of the centers by the learner).
    \item The learner suffers the connection cost of the clients,
    \[C_{R_t}(F_t) = \left(\sum_{j \in R_t} d(j,F_t)^p\right)^{1/p}\]
    where $d(j,F_t)$ is the distance of client $j$ to the closest center, $d(j,F_t) = \min_{i \in F_t}d_{ij}$.
\end{enumerate}
\end{online_problem}
Based on the past positions of the clients $R_1,R_2,\ldots, R_{t-1}$ an \textit{online learning algorithm} must select at each round $t$, a set of $k$ centers $F_t \subseteq V$ such that the 
connection cost of the clients over time is close to the connection cost of the \textit{optimal (static) solution} $F^\ast$. If the cost of the online learning algorithm is at most $\alpha$ times the cost of $F^\ast$, the algorithm is called $\alpha$-regret, whereas in case $\alpha = 1$, the algorithm is called \textit{no-regret} \cite{H16}. Intuitively, a low-regret online learning algorithm 
converges to the optimal positions of the centers (with respect to the overall trajectories of the clients) by just observing the clients' dynamics.

\begin{example}\label{ex:1}
The clients are randomly generated according to a time-varying uniform distribution with radius $0.3$ and center following the periodic trajectory $\left(\sin ( \frac{2\pi \cdot t}{T}),\cos ( \frac{2\pi \cdot t}{T})\right)$ for $t=1,\ldots,T$.
\begin{figure}[!htb]
\centering
  {\includegraphics[width=0.7\linewidth]{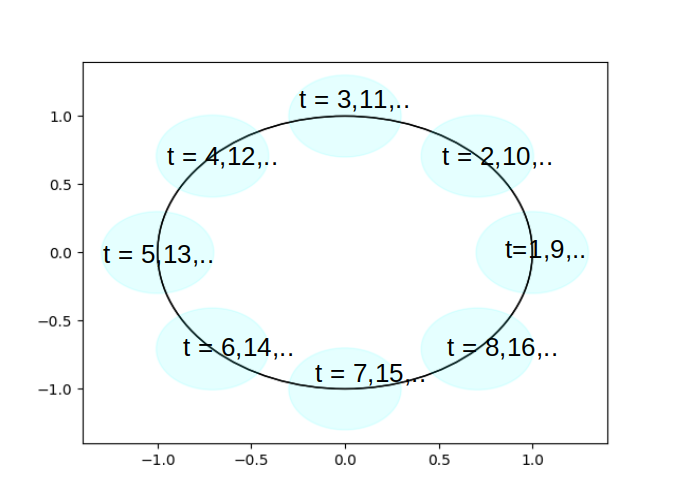}\label{fig:sub1}}\hfill
  \label{f:circle}
\end{figure}
The centers placed by a (sufficiently) low-regret algorithm would converge to positions similar in structure to the ones illustrated in Figure~\ref{f:circle2} (for $k=1,2,4$ and $k=8$) which are clearly close to the optimal (static) solution for the different values of $k$.
\begin{figure}[!htb]
\centering
 {\includegraphics[width=0.45\linewidth]{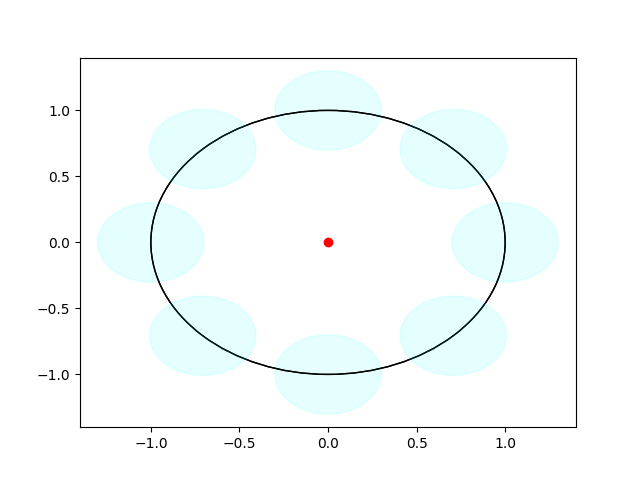}\label{fig:sub1}}\hfill
 {\includegraphics[width=0.45\linewidth]{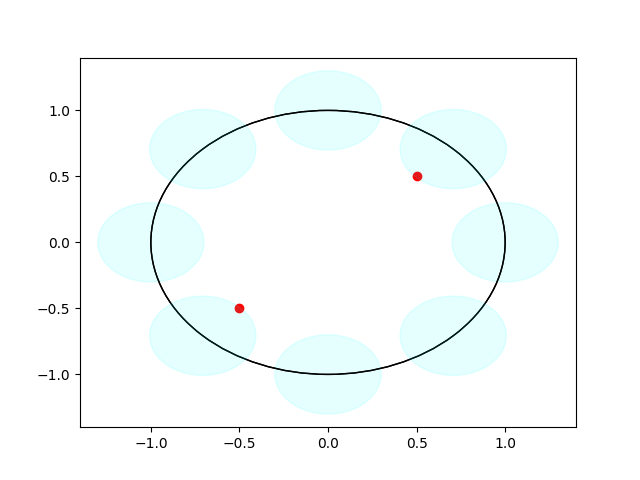}\label{fig:sub2}}\hfill
{\includegraphics[width=0.45\linewidth]{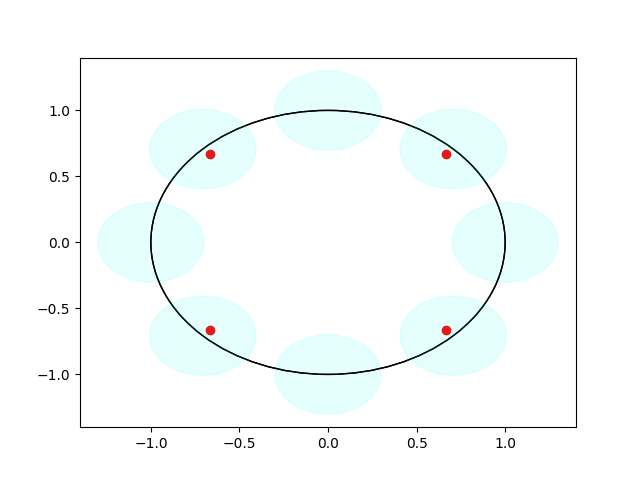}\label{fig:sub2}}\hfill
{\includegraphics[width=0.45\linewidth]{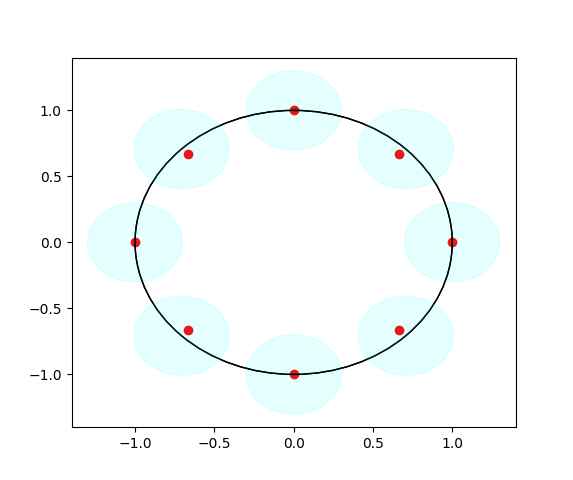}\label{fig:sub2}}\hfill
\caption{The figure depicts the actual centers at which a low-regret algorithm, that we subsequently propose, converges. For further details see Section~\ref{s:experiments}.}
\label{f:circle2}
\end{figure}
\end{example}

\textbf{Efficient Online Learning for Dynamic $k$-Clustering.}
The existence of no-regret online learning algorithms for Dynamic $k$-Clustering immediately follows by standard results in online learning literature \cite{H16}. Dynamic $k$-Clustering is a special case of \textit{Learning from Expert Advice} problem for which the famous \textit{Multiplicative Weights Update Algorithm} achieves no-regret \cite{H16}. Unfortunately using the $\mathrm{MWU}$ for Dynamic $k$-Clustering is not really an option due to the huge time and space complexity that $\mathrm{MWU}$ requires. In particular $\mathrm{MWU}$ keeps a different weight (probability) for each of the possible ${|V|}\choose{k}$ possible placements of the $k$ centers, rendering it inapplicable even for small values of $|V|$ and $k$.

Our work aims to shed light on the following question.
\begin{question}\label{q:main}
Is there an online learning algorithm
for Dynamic $k$-Clustering that runs in polynomial time and achieves $\alpha$-regret?
\end{question}

\smallskip
\textbf{Our Contribution and Techniques.}
%
We first show that constant regret cannot be achieved in polynomial time for Dynamic $k$-Clustering. In particular we prove that any $O(1)$-regret polynomial-time online learning algorithm for Dynamic $k$-Clustering 
implies the existence of an $O(1)$-approximation algorithm for the \textit{Minimum-$p$-Union problem} \cite{CDKKR16}. Recent works on the theory of computational complexity establish that unless well-established cryptographic conjectures fail, there is no $O(1)$-approximation algorithm for $\mathrm{Min}$-$p$-$\mathrm{Union}$ \cite{CDKKR16,A12,CDM17}. This result narrows the plausible regret bounds achievable in polynomial time, and reveals an interesting gap between Dynamic $k$-Clustering and its offline counterparts, which admit polynomial-time $O(1)$-approximation algorithms.

Our main technical contribution consists of polynomial-time online learning algorithms for Dynamic $k$-Clustering with non trivial regret bounds. We present a $\Theta(k)$-regret polynomial-time deterministic online learning algorithm and a $\Theta(r)$-regret polynomial-time randomized online learning algorithm,
where $r$ is the maximum number of clients
appearing in a single round ($r = \max_{1\leq t \leq T}|R_t|$). Combining these algorithms, one can achieve $\Theta\left( \min(k,r) \right)$-regret for Dynamic $k$-Clustering, which (to the best of our knowledge) is the first guarantee on the regret achievable in polynomial time. The regret bounds above are independent of the selected $p$-norm, and hold for any $p \geq 1$ and for $p = \infty$.

At a technical level, our approach consists of two major steps. In the first step, we consider an online learning problem, that can be regarded as the \textit{fractional relaxation} of the Dynamic $k$-Clustering (see Section~\ref{s:fractional}), where the \textit{fractional connection cost} is given by the optimal value of an appropriate convex program and 
the action space of the learner is the $|V|$-dimensional simplex. For this intermediate problem, we design a \textit{no-regret} polynomial-time online learning algorithm through the use of the subgradients of the fractional connection cost. We show that such subgradient vectors can be computed in polynomial time via the solution of
the dual program of the fractional connection cost. In the second step of our approach (see Section~\ref{s:det} and Section~\ref{s:rand}), we provide computationally efficient online (deterministic and randomized) rounding schemes converting a vector lying in the 
$|V|$-dimensional simplex (the action space of Fractional Dynamic $k$-Clustering) into $k$ locations for the centers on the metric space $V$ (the action space of Dynamic $k$-Clustering).

In Section~\ref{s:det}, we present a deterministic rounding scheme that, combined with the no-regret algorithm for Fractional Dynamic $k$-Clustering, leads to a $\Theta(k)$-regret polynomial-time deterministic online learning algorithm for the original Dynamic $k$-Clustering. 
Interestingly, this regret bound is approximately optimal for all deterministic
algorithms. In Section~\ref{s:rand}, we show that 
combining the no-regret algorithm for Fractional Dynamic $k$-Clustering with a
randomized rounding scheme proposed in \cite{CL12}\footnote{This randomized rounding scheme was part of a $4$-approximation algorithm for $k$-median \cite{CL12}} leads to a $\Theta(r)$-regret randomized algorithm running in polynomial time. Combining these two online learning algorithms, we obtain a $\Theta(\min(k,r))$-regret polynomial-time online learning algorithm for Dynamic $k$-Clustering, which is the main technical contribution of this work. Finally, in Section~\ref{s:experiments}, we present the results of an experimental evaluation, indicating that for client locations generated in a variety of natural and practically relevant ways, the realized regret of the proposed algorithms is way smaller than $\Theta\left( \min(k,r) \right)$.

\begin{remark}
Our two-step approach provides a structured framework for designing polynomial-time low-regret algorithms in various combinatorial domains. The first step 
extends far beyond the context of Dynamic $k$-Clustering and provides a systematic approach to the design of \textit{polynomial-time no-regret online learning algorithms} for the \textbf{fractional relaxation} of the combinatorial online learning problem of interest.
Combining such no-regret algorithms with 
online rounding schemes, which convert fractional solutions into integral solutions of the original online learning problem, may lead to polynomial time low-regret algorithms for various combinatorial settings. Obviously, designing such rounding schemes is usually far from trivial, since the specific combinatorial structure of each specific problem must be taken into account.
%
\end{remark}

\textbf{Related Work.}
Our work relates with the research line of Combinatorial Online Learning. There exists a long line of research studying low-regret online learning algorithms for various combinatorial domains such that
online routing \cite{HS97,AK08}, selection of permutations \cite{TW00,YHKSTT11,FLPS20,A14,HW07}, selection of binary search trees \cite{TM03}, submodular optimization \cite{HK12a,JB11,SG08}, matrix completion \cite{HKS12},
contextual bandits \cite{ALLS14,DHKKLRZ11} and many more. Finally, in combinatorial games agents need to learn to play optimally against each other over complex domains \cite{ITLMPT11,dehghani2016price}.
As in the case of Dynamic $k$-Clustering in all the above online learning problems, MWU is not an option, due to the exponential number of possible actions.  

Another research direction of Combinatorial Online Learning studies \textit{black-box reductions} converting polynomial time offline algorithm (full information on the data) into polynomial time online learning algorithms. \cite{kalai03} showed that any (offline) algorithm solving optimally and in polynomial time the objective function, that the \textit{Follow the Leader framework} suggests, can be converted into a no-regret online learning algorithm. \cite{kakade07} extended the previous result for specific class of online learning problems called \textit{linear optimization problems} for which they showed that any $\alpha$-approximation (offline) can be converted into an $\alpha$-regret online learning algorithm. They also provide a surprising counterexample showing that such black-box reductions do not hold for general combinatorial online learning problems. Both the 
time efficiency and the regret bounds of the reductions of \cite{kalai03} and \cite{kakade07} were subsequently improved by \cite{rahmanian17,suehiro12,koolen10,balcan06,syrganis17,hazan16,fujita13,garber17,wei18}. We remark that
the above results do not apply in our setting since 
Dynamic $k$-Clustering can neither be optimally solved in polynomial-time nor is a linear optimization problem.

Our works also relates with the more recent line of research studying clustering problems with \textit{time-evolving clients}. \cite{EMS14} and \cite{ANS17} respectively provide $\Theta\left( \log (nT)\right)$ and $O(1)$-approximation algorithm for a generalization of the facility location problem in which clients change their positions over time. The first difference of Dynamic $k$-Clustering with this setting is that 
in the former case there is no constraint on the number of centers that can open 
 and furthermore, crucially
perfect knowledge of the positions of the clients is presumed. 
More closely related to our work are \cite{KW18,FKKLSZ19}, where the special case of Dynamic $k$-Clustering on a line is studied (the clients move on a line over time). Despite the fact that both works study online algorithms, which do not require knowledge on the clients' future positions, they only provided positive results for $k=1$~and~$2$.

\section{Preliminaries and Our Results}\label{s:prelim}
In this section we introduce notation and several key notions as long as present the formal Statements of our results. 

We denote by $D$ the diameter of the metric space, $D = \max_{i \in V, j \in V} d_{ij}$. We denote with $n$ the cardinality of the metric space $\left(|V| =n\right)$ and with $r$ the maximum number of clients appearing in a single round, $r = \max_{1\leq t \leq T}|R_t|$. Finally we denote with $\Delta_{n}^k$ the $n$-dimensional simplex, $\Delta_{n}^k= \{y \in \mathbb{R}^n:~ \sum_{i \in V} y_i = k ~\mathrm{and}~y_i \geq 0\}$.

Following the standard notion of regret in online learning \cite{H16}, we provide the formal definition of an \textit{$\alpha$-regret} online learning algorithm for \textit{Dynamic $k$-Clustering}.
\begin{definition}\label{d:regret}
An online learning algorithm for the \textit{Dynamic $k$-Clustering} is $\alpha$-regret if and only if for any sequence of clients' positions $R_1,\ldots,R_T \subseteq V$,

\[\sum_{t=1}^T C_{R_t}(F_t) \leq \alpha \cdot  \min_{|F^\ast| \leq k} \sum_{t=1}^T C_{R_t}(F^\ast) + \Theta\left(\mathrm{poly}(n,D) \cdot T^\beta\right)\]
where $F_1,\ldots,F_T$ are the positions of the centers produced by the algorithm for the sequence $R_1,\ldots,R_T$ and $\beta < 1$. 
\end{definition}

Next, we introduce the \textit{Minimum-$p$-Union} problem,
the inapproximability results of which allow us to establish that constant regret cannot be achieved in polynomial time for Dynamic $k$-Clustering.
\begin{problem}[$\mathrm{Min-}p\mathrm{-Union}$]
Given a universe of elements $\mathbb{E}$ and a collection of sets
$\mathbb{U} =\{S_1, \ldots, S_m\}$ where $S_i \subseteq \mathbb{E}$. Select $\mathbb{U}' \subseteq \mathbb{U}$ such that $|\mathbb{U'}| =p$ and $|\cup_{S_i \in \mathbb{U}'}S_i|$ is minimized.
\end{problem}

As already mentioned, the existence of an $O(1)$-approximation algorithm for $\mathrm{Min-}p\mathrm{-Union}$ violates several widely believed conjectures in computational complexity theory\cite{CDKKR16,A12,CDM17}. In Theorem~\ref{t:hardnes} we establish the fact that the exact same conjectures are violated in case there exists an online learning algorithm for \textit{Dynamic $k$-Clustering} that runs in polynomial-time and achieves $O(1)$-regret. 
\begin{theorem}\label{t:hardnes}
Any $c$-regret polynomial-time online learning algorithm for the Dynamic $k$-Clustering implies a $(c+1)$-approximation polynomial-time algorithm for $\mathrm{Min-}p\mathrm{-Union}$.
\end{theorem}
In Section~\ref{s:det}, we present a polynomial-time deterministic online learning algorithm achieving \textit{$\Theta(k)$}-regret.

\begin{theorem}\label{t:det-regret}
There exists a $6k$-regret deterministic online learning algorithm for Dynamic $k$-Clustering that runs in polynomial time (Algorithm~\ref{alg:det}). More precisely,
\[\sum_{t=1}^T\mathrm{C}_{R_t}(F_t) \leq  6k \cdot \min_{|F^\ast|=k} \sum_{t=1}^T\mathrm{C}_{R_t}(F^\ast) + \Theta \left (k D n \sqrt{\log n  T} \right)\]
where $F_1,\ldots,F_T$ are the positions in which Algorithm~\ref{alg:det} places the centers for the sequence of clients' positions $R_1,\ldots,R_T$.
\end{theorem}
In Theorem~\ref{t:lower_bound_det} we prove that the $\Theta(k)$ bound on the regret of Algorithm~\ref{alg:det} cannot be significantly ameliorated with deterministic online learning algorithm even if the algorithm uses exponential time and space.

\begin{theorem}\label{t:lower_bound_det}
For any deterministic online learning algorithm for Dynamic $k$-Clustering problem, there exists a sequence of clients $R_1,\ldots,R_T$ such as the regret is at least $k+1$.
\end{theorem}

In Section~\ref{s:rand} we present a randomized online learning algorithm the regret of which depends on the parameter $r$.

\begin{theorem}\label{t:rand-regret}
There exists a $\Theta(r)$-regret randomized algorithm that runs in polynomial time (Algorithm~\ref{alg:rand}).
For any sequence of clients' positions $R_1,\ldots,R_T$ with $|R_t| \leq r$,
\begin{equation*}
\begin{split}
\sum_{t=1}^T \mathbb{E}\left[C_{R_t}(F_t)\right] & = 4r \cdot \min_{|F^\ast|=k} \sum_{t=1}^T\mathrm{C}_{R_t}(F^\ast)\\ &+ \Theta \left (k D n \sqrt{\log n  T} \right)
\end{split}
\end{equation*}
where $F_t$ is the random variable denoting the $k$ positions at which Algorithm~\ref{alg:rand} places the centers at round $t$.
\end{theorem}

By combining Algorithm~\ref{alg:det} and Algorithm~\ref{alg:rand} we can achieve $\Theta \left(\min(k,r)\right)$-regret in polynomial time.

\begin{theorem}\label{t:main}
There exists an online learning algorithm for Dynamic $k$-Clustering that runs in polynomial-time and achieves $\min\left(6k,4r \right)$-regret.
\end{theorem}
\begin{remark}
In case the value $r = \min_{1\leq t \leq T}|R_t|$ is initially known to the learner, then Theorem~\ref{t:main} follows directly by Theorem~\ref{t:det-regret}~and~\ref{t:rand-regret}. However even if $r$ is not initially known, the learner can run a Multiplicative Weight Update Algorithm that at each round follows either Algorithm~\ref{alg:det} or Algorithm~\ref{alg:rand} with some probability distribution depending on the cost of each algorithm so far. By standard results for MWU \cite{H16}, this meta-algorithm admits time-average cost less than the best of Algorithm~\ref{alg:det}~and~\ref{alg:rand}.
\end{remark}
\section{Fractional Dynamic $k$-Clustering
}\label{s:fractional}
In this section we present the \textit{Fractional Dynamic $k$-Clustering} problem for which we provide a polynomial-time no-regret online learning algorithm. This online learning algorithm serves as a primitive for both Algorithm~\ref{alg:det} and Algorithm~\ref{alg:rand} of the subsequent sections concerning the original Dynamic $k$-Clustering.

The basic difference between Dynamic $k$-Clustering and Fractional Dynamic $k$-Clustering is that in the second case the learner can \textit{fractionally} place a center at some point of the metric space $V$. Such a fractional opening is described by a vector $y \in \Delta_{n}^k$.

\begin{online_problem}\label{pr:frac}
[Fractional Dynamic $k$-Clustering]At each round $t \geq 1$, 
\begin{enumerate}
    \item The learner selects a vector $y_t \in \Delta_{n}^k$. The value $y_i^t$ stands for the fractional amount of center that the learner opens in position $i \in V$. 
    
    \item The adversary selects the positions of the clients denoted by $R_t \subseteq V$ (after the selection of the vector $y_t$).
    
        \item The learner incurs fractional connection cost $\mathrm{FC}_{R_t}(y_t)$ described in Definition~\ref{d:frac_cost}.
\end{enumerate}
\end{online_problem}

\begin{definition}[Fractional Connection Cost]\label{d:frac_cost}
Given the positions of the clients $R \subseteq V$, we define the fractional connection cost $\mathrm{FC}_{R}(\cdot)$ of
a vector $y \in \Delta_n^k$ as
the optimal value of the following convex program.
\begin{equation}
\begin{array}{lr@{}ll}
\mbox{\emph{minimize}}  \left(\sum_{j \in R}\beta_j^p \right)^{1/p}
\\
\\
\mathrm{}{s.t.}~~~~ \beta_j = \sum\limits_{i \in V} d_{ij} \cdot x_{ij} \,\,~~~~\forall j \in R\\
~~~~~~~~~ \sum\limits_{i \in V} x_{ij} = 1 \,\,~~~~~~~~~~~~~~~\forall j \in R\\
~~~~~~~~  x_{ij} \leq y_i \,\,~~~~~~~\forall j \in R,~\forall i \in V\\
~~~~~~~~  x_{ij} \geq 0 \,\,~~~~~~~~\forall j \in R,~\forall i \in V
\end{array}
\end{equation}
\end{definition}
It is not hard to see that once the convex program of Definition~\ref{d:frac_cost} is formulated with respect to an \textit{integral vector} $y\in \Delta_n^k$ ($y_i$ is either $0$ or $1$) the 
fractional connection cost $\mathrm{FC}_{R}(y)$ equals the original connection cost $\mathrm{C}_{R}(y)$. As a result, the cost of the optimal solution $y^\ast \in \Delta_{k}^n$ of Fractional Dynamic $k$-Clustering is upper bounded by the cost of the optimal positioning of the centers $F^\ast$ in the original Dynamic $k$-Clustering.

\begin{lemma}\label{l:frac_int}
For any sequence of clients' positions $R_1,\ldots,R_T$, the cost of the optimal fractional solution $y^\ast$ for Fractional Dynamic $k$-Clustering
is smaller than the cost of the optimal positioning $F^\ast$ for Dynamic $k$-Clustering,
\[ \min_{y^\ast \in \Delta_n^k}\sum_{t=1}^T \mathrm{FC}_{R_t}(y^\ast) \leq \min_{ |F^\ast| =k}\sum_{t=1}^T \mathrm{C}_{R_t}(F^\ast)\]
\end{lemma}

Lemma~\ref{l:frac_int} will be used in the next sections where the online learning algorithms for the original Dynamic $k$-Clustering are presented. To this end, we  dedicate the rest of this section to design a polynomial time no-regret algorithm for Fractional Dynamic $k$-Clustering. A key step towards this direction is the use of the subgradient vectors of $\mathrm{FC}_{R_t}(\cdot)$.
\begin{definition}[Subgradient]\label{d:subgradients}
Given a function $f:\mathbb{R}^n \mapsto \mathbb{R}$, a vector $g \in \mathbb{R}^n$ belongs in the subgradient of $f$ at point $x\in \mathbb{R}^n$,$g \in \partial f(x)$, if and only if $f(y) \geq f(x) + g^\top (y -x)~$, for all $y \in \mathbb{R}^n$.
\end{definition}
Computing the subgradient vectors of functions, as complicated as $\mathrm{FC}_{R_t}(\cdot)$, is in general a computationally hard task. One of our main technical contributions consists in showing that the latter can be done through the solution of an adequate convex program corresponding to the dual of the convex program of Definition~\ref{d:frac_cost}.  
\begin{lemma}\label{l:dual}
Consider the convex program of Definition~\ref{d:frac_cost} formulated with respect to a vector $y \in  \Delta_n^k$ and the clients' positions $R$. Then the following convex program is its dual.
\begin{equation}\label{eq:ALP}
\begin{array}{lr@{}ll}
\mbox{\emph{maximize}}~~~ \sum_{j \in R}A_j - \sum_{i \in V}\sum_{j \in R}k_{ij}\cdot y_i\\
\\
\mathrm{s.t.}~~~~ ||\lambda||_{p}^\ast \leq 1 \\
~~~~~~~~~  d_{ij} \cdot  \lambda_j + k_{ij} \geq A_j \,\,~~~~\forall i \in V, j \in R\\
~~~~~~~~~~k_{ij} \geq 0 \,\,~~~~~~~~~~~~~~~~~~~~~~~\forall i \in V, j \in R\\
\end{array}
\end{equation}
where $|| \cdot||_{p}^\ast$ is the dual norm of $||\cdot ||_p$
\end{lemma}
In the following lemma we establish the fact that a subgradient vector of $\partial \mathrm{FC}_{R_t}(\cdot)$ can be computed through the optimal solution of the convex program in Lemma~\ref{l:dual}.
\begin{lemma}\label{l:subgradients}
Let $k_{ij}^\ast$ denote the value of the variables $k_{ij}$ in the optimal solution of the convex program in Lemma~\ref{l:dual} formulated with respect to vector $y \in \Delta_n^k$ and the clients' positions $R$. Then for any vector $y' \in \Delta_n^k$,
\[\mathrm{FC}_{R_t}(y') \geq \mathrm{FC}_{R_t}(y) + \sum_{i \in V} \left(-\sum_{j \in R}k_{ij}^{\ast} \right)\cdot \left( y_i' - y_i \right) \]
Moreover there exits an $\Theta(r \cdot |V|)$ algorithm for solving the dual program (Algorithm~\ref{alg:dual}) and additionally $|k_{ij}^\ast| \leq D$. 
\end{lemma}

\begin{algorithm}[H]
  \caption{A time-efficient algorithm for solving the dual program of Lemma~\ref{l:dual}}
  \begin{algorithmic}[1]
  \State\textbf{Input:} A vector $y \in \Delta_{n}^k$ and a set of clients $R \subseteq V$.
  
  \State \textbf{Output:} An optimal solution for the convex program of Lemma~\ref{l:dual}.

  \For{ each client $j \in R$,}
  \State Sort the nodes $i \in V$ in increasing order according to $d_{ij}$. 
  
    \State $\mathrm{Rem}  \leftarrow 1$  
  
   \For{each each $i \in V$}
        \State $x_{ij} \leftarrow \min(y_i , \mathrm{Rem})$.

        \State $\mathrm{Rem} \leftarrow \mathrm{Rem} - x_{ij}$.
   \EndFor
  \EndFor

  \For{ each client $j \in R$}
  \State $V_j^+ \leftarrow \{i \in V:~ x_{ij} > 0\}$  and $D_j \leftarrow \max_{i \in V_{j}^+} d_{ij}$. 
  
    \State $\beta_j \leftarrow \sum_{i \in V}d_{ij} \cdot x_{ij}$

  \State $\lambda_j \leftarrow \left[ \frac{\beta_j}{||\beta||_p} \right]^{p-1}$
  
  \State $A_j \leftarrow \lambda_j \cdot D_j$
  
  \State $k_{ij} \leftarrow \min \left(\lambda_j\cdot \frac{x_{ij}}{y_i} \cdot \left(D_j - d_{ij}\right) , 0 \right)$
  
  \EndFor
  \end{algorithmic}
  \label{alg:dual}
\end{algorithm}
\begin{remark}
Algorithm~\ref{alg:dual} is not only a computationally efficient way to solve the convex program of Lemma~\ref{l:dual}, but most importantly guarantees that the value $k_{ij}^\ast$ are bounded by $D$ (this is formally Stated and proven in Lemma~\ref{l:dual}). The latter property is crucial for developing the no-regret algorithm for Fractional Dynamic $k$-Clustering. 
\end{remark}
Up next we present the no-regret algorithm for Fractional Dynamic $k$-Clustering.

\begin{algorithm}[H]
  \caption{A no-regret algorithm for Fractional Dynamic $k$-Clustering}
  \begin{algorithmic}[1]
  \State Initially, the learner selects $y^1_i = k/n$ for all $i \in V$.
  \For{ rounds $t = 1 \cdots T$}
    \State The learner selects $y_t \in \Delta_{n}^k$.
  \State The adversary selects the positions of the clients $R_t \subseteq V$.
  \State The learner receives cost, $\mathrm{FC}_{R_t}(y_t)$.
  
  \State The learner runs Algorithm~\ref{alg:dual} with input $y_t$ and $R_t$ and sets $g_i^t = -\sum_{ j \in R_t}k_{ij}^t$
  
  \For{ each $i \in V$}

\State
  \[y_i^{t+1} = \frac{ y_i^t \cdot e^{-\epsilon g_i^t}}{\sum_{i\in V}y_i^t \cdot e^{- \epsilon g_i^t}}\]
  where $\epsilon =\frac{\sqrt{ \log n}}{D r \sqrt{T}}$
  \EndFor
\EndFor
  \end{algorithmic}
  \label{alg:frac_no_regret}
\end{algorithm}

We conclude the section with Theorem~\ref{t:no-regret-frac} that establishes the no-regret property of Algorithm~\ref{alg:frac_no_regret} and the proof of which is deferred to the Appendix~\ref{app:fractional}. 

\begin{theorem}\label{t:no-regret-frac}
Let $y_1,\ldots,y_T$ be the sequence of vectors in $\Delta_n^k$ produced by Algorithm~\ref{alg:frac_no_regret} for the clients' positions $R_1,\ldots,R_T$. Then,
\[\sum_{t=1}^T\mathrm{FC}_{R_t}(y_t) \leq  \min_{y^\ast \in \Delta_k} \sum_{t=1}^T\mathrm{FC}_{R_t}(y^\ast) + \Theta \left (k D n \sqrt{\log n  T} \right)\]
\end{theorem}

\section{A $\Theta(k)$-Regret Deterministic Online Learning Algorithm}\label{s:det}
In this section we show how one can use Algorithm~\ref{alg:frac_no_regret} described in Section~\ref{s:fractional} to derive $\Theta(k)$-regret for the Dynamic $k$-Clustering in polynomial-time.

The basic idea is to use a rounding scheme that given a vector $y\in \Delta_n^k$ produces a placement of the $k$ centers $F_y \subseteq V$ (with $|F_y| \leq k$) such that \textit{for any set of clients' positions $R$}, the 
connection cost $C_{R}(F_y)$ is approximately bounded by the factional connection cost $\mathrm{FC}_{R}(y)$. This rounding scheme is described in Algorithm~\ref{alg:rounding}.

\begin{algorithm}
  \caption{Deterministic Rounding Scheme}
  \begin{algorithmic}[1]
  \State \textbf{Input}: A vector $y \in \Delta_{n}^k$.
 \State \textbf{Output}: A set $F_y \subseteq V$ at which centers are opened.

  \State Run Algorithm~\ref{alg:dual} with input $y$ and $R = V$.
  
  \State Sort the positions $i \in V$ according to the values $\beta_i$ produced by Algorithm~\ref{alg:dual}.
  \State $F_y \leftarrow \emptyset$
 
  \For{ $i = 1$ \bfseries{ to } $V$}
  \If{ $\min _{j \in F_y}d_{ij} > 6k \cdot \beta_i$}
    \State $F_y \leftarrow F_y \cup \{i\}$
    \EndIf
      \EndFor
  \end{algorithmic}
  \label{alg:rounding}
\end{algorithm}

\begin{lemma}\label{l:rounding_lemma}[Rounding Lemma] Let $F_y$ denote the positions of the centers produced by Algorithm~\ref{alg:rounding} for input $y \in \Delta_n^k$. Then the following properties hold,

\begin{itemize}
    \item For any set of clients $R$,
    \[\mathrm{C}_R(F_y)~\leq 6k \cdot \mathrm{FC}_{R}(y)\]
    
    \item The cardinality of $\mathrm{F}_y$ is at most $k$, $|\mathrm{F}_y| \leq k$.
\end{itemize}
\end{lemma}

Up next we show how the deterministic rounding scheme described in Algorithm~\ref{alg:rounding} can be combined with Algorithm~\ref{alg:frac_no_regret} to produce an $\Theta(k)$-regret deterministic online learning algorithm that runs in polynomial-time. The overall online learning algorithm is described in Algorithm~\ref{alg:det} and its regret bound is formally Stated and proven in Theorem~\ref{t:det-regret}.

\begin{algorithm}[H]
  \caption{A $\Theta(k)$-regret deterministic online learning algorithm
  for Dynamic $k$-Clustering}
  \label{alg:det}
  
  \begin{algorithmic}[1]
  \For{ rounds $t = 1 \cdots T$}
  
  \State The learner computes the vector $y_t \in \Delta_{n}^k$ by running Algorithm~\ref{alg:frac_no_regret} for the sequence of clients' positions $(R_1,\ldots,R_{t-1})$.
  
  \State The learner places centers to the positions $F_{y_t}$ produced by Algorithm~\ref{alg:rounding} given input $y_t$.
  
  \State The adversary selects the clients' positions $R_t \subseteq V$.
  
   \State The learner suffers connection cost $C_{R_t}(F_{y_t})$

  \EndFor
  \end{algorithmic}
  \label{alg:det}
\end{algorithm}

We conclude the section with the proof of Theorem~\ref{t:det-regret} in which the regret bounds of Algorithm~\ref{alg:det} are established.
\begin{proof}[Proof of Theorem~\ref{t:det-regret}]
The second case of Lemma~\ref{l:rounding_lemma} ensures that $|F_{t}|\leq k$ and thus Algorithm~\ref{alg:det} opens at most $k$ facilities at each round. Applying the first case of Lemma~\ref{l:rounding_lemma} for $R=R_t$ we get that $C_{R_t}(F_t)\leq 6k \cdot \mathrm{FC}_{R_t}(y_t)$. As a result,
\begin{eqnarray*}
&&\sum_{t=1}^T C_{R_t}(F_t) \leq \sum_{t=1}^T 6k \cdot \mathrm{FC}_{R_t}(y_t)\\
&&\leq 6k \min_{y^\ast \in \Delta_k}
\sum_{t=1}^T \mathrm{FC}_{R_t}(y^\ast) + \Theta \left (k D n \sqrt{\log n  T} \right)
\end{eqnarray*}
where the last inequality follows by Theorem~\ref{t:no-regret-frac}. However Lemma~\ref{l:frac_int} ensures that \[\min_{y^\ast \in \Delta_k} \sum_{t=1}^T \mathrm{FC}_{R_t}(y^\ast) \leq \min_{F^\ast: |F^\ast|=k} \sum_{t=1}^T \mathrm{C}_{R_t}(F^\ast)\]
\end{proof}


\section{A \textbf{$\Theta(r)$}-Regret Randomized Online Learning Algorithm}\label{s:rand}

In this section we present a $\Theta(r)$-regret
randomized online learning algorithm.
This algorithm is described in Algorithm~\ref{alg:rand} and is based on the randomized rounding developed by Charikar and Li for the $k$-median problem \cite{CL12}.

\begin{lemma}[\cite{CL12}]\label{l:Charikar-Lin}
There exists a polynomial-time randomized rounding scheme that given a vector $y \in \Delta_n^k$
produces a probability distribution, denoted as $\mathrm{CL}(y)$, over the subsets of $V$ such that,
\begin{enumerate}
    \item with probability $1$ exactly $k$ facilities are opened, $\mathbb{P}_{F \sim \mathrm{CL}(y)}\left[|F| = k\right] = 1$.
    
    \item for any position $j \in V$,
    \[\mathbb{E}_{F \sim \mathrm{CL}(y)}\left[C_{\{j\}}(F_y)  \right] \leq 4 \cdot \mathrm{FC}_{\{j\}}(y).\]
\end{enumerate}
\end{lemma}
Similarly with the previous section, combining the randomized rounding of Charikar-Li with Algorithm~$1$ produces a $\Theta(r)$-regret randomized online learning algorithm that runs in polynomial-time.

\begin{algorithm}[H]
  \caption{A $\Theta(r)$-regret randomized online learning algorithm}
  \label{alg:rand}
  
  \begin{algorithmic}[1]
  
  \For{ rounds $t = 1 \cdots T$}
  
  \State The learner computes the vector $y_t \in \Delta_{n}^k$ by running Algorithm~\ref{alg:frac_no_regret} for the sequence of clients' positions $(R_1,\ldots,R_{t-1})$.
  
  \State The learner places centers to the positions  $F_t \subseteq V$ produced by the Charikar-Li randomized rounding with input $y_t$, $F_t \sim \mathrm{CL}(y_t)$.
  
  \State The adversary selects a request $R_t \subseteq V$.
  
   \State The learner suffers connection cost $C_{R_t}(F_t)$

  \EndFor
  \end{algorithmic}
\end{algorithm}

The proof of Theorem~\ref{t:rand-regret} that establishes the regret bound of Algorithm~\ref{alg:rand} follows by Lemma~\ref{l:Charikar-Lin} and Theorem~\ref{t:no-regret-frac} and is deferred to the Appendix~\ref{app:rand}.
\section{Experimental Evaluations}\label{s:experiments}
In this section we evaluate the performance of our online learning algorithm against adversaries that select the positions of the clients according to time-evolving probability distributions. We remark that the
regret bounds established in Theorem~\ref{t:det-regret} and Theorem~\ref{t:rand-regret} hold even if the adversary \textit{maliciously} selects the positions of the clients at each round so as to maximize the connection cost. As a result, in case clients arrive according to some (unknown and possibly time-varying) probability distribution that does not depend on the algorithm's actions, we expect the regret of to be way smaller.

In this section we empirically evaluate the regret of Algorithm~\ref{alg:det} for Dynamic $k$-Clustering in case $p =\infty$. We assume that at each round $t$, $20$ clients arrive according to several static or time-varying two-dimensional probability distributions with support on the $[-1,1] \times [-1,1]$ square and the possible positions for the centers being the discretized grid with $\epsilon = 0.1$. In order to monitor the quality of the solutions produced by Algorithm~\ref{alg:det}, we compare the time-average connection cost of Algorithm~\ref{alg:det} with the time-average \textit{fractional connection cost} of Algorithm~\ref{alg:frac_no_regret}. Theorem~\ref{t:no-regret-frac} ensures that for $T=\Theta(k^2 D^2/\epsilon^2)$ the time-average fractional connection cost of Algorithm~\ref{alg:frac_no_regret} is at most $\epsilon$
greater than the time-average connection cost of the optimal static solution for Dynamic $k$-Clustering. In the following simulations we select $\epsilon = 0.1$ and track the ratio between the time-average cost of Algorithm~\ref{alg:det} and of Algorithm~\ref{alg:frac_no_regret} which acts as an upper bound on the regret.

\textbf{Uniform Square} In this case the $20$ clients arrive \textit{uniformly at random} in the $[-1,1] \times [-1,1]$ square. Figure~\ref{f:uniform_square} illustrates the solutions at which Algorithm~\ref{alg:det} converges for $k=2,3$ and $8$ as long as the achieved regret. 

\begin{figure}[!htb]
\centering
  {\includegraphics[width=0.49\linewidth]{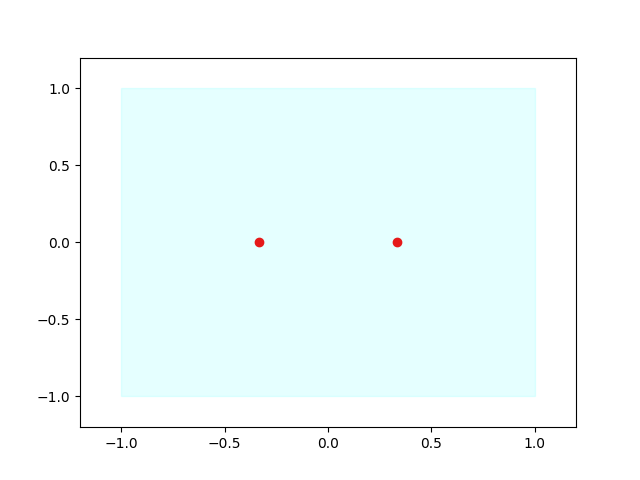}\label{fig:sub2}}\hfill
  {\includegraphics[width=0.49\linewidth]{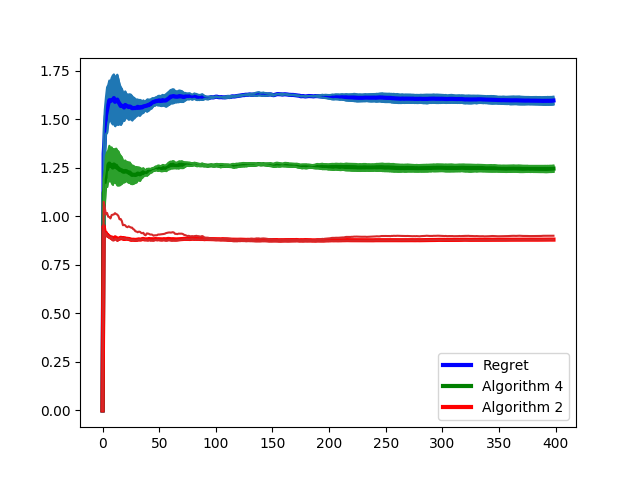}\label{fig:sub3}}\hfill
{\includegraphics[width=0.49\linewidth]{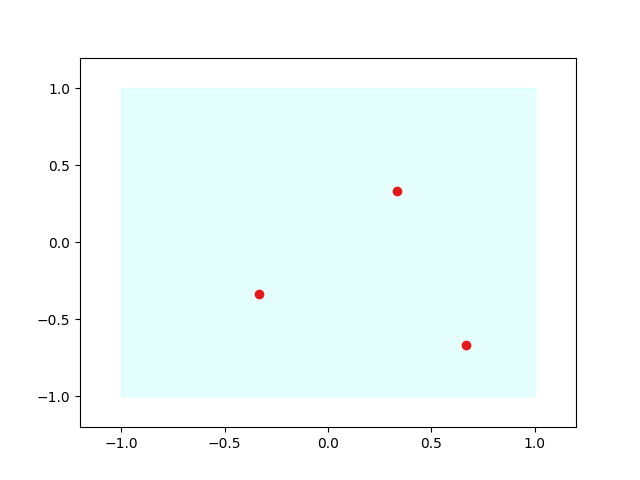}\label{fig:sub2}}\hfill
 {\includegraphics[width=0.49\linewidth]{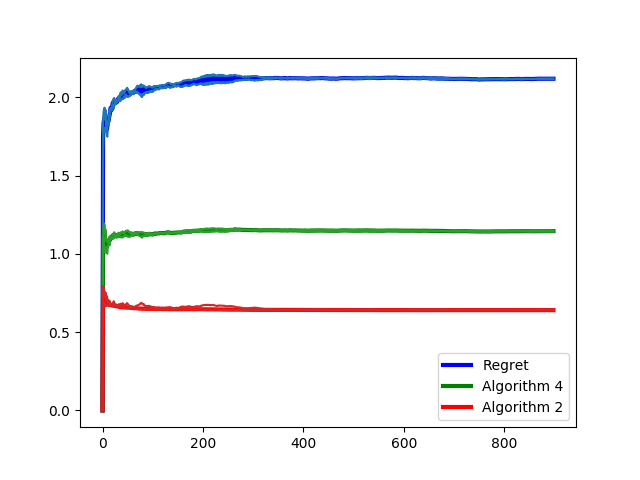}\label{fig:sub3}}\hfill
{\includegraphics[width=0.49\linewidth]{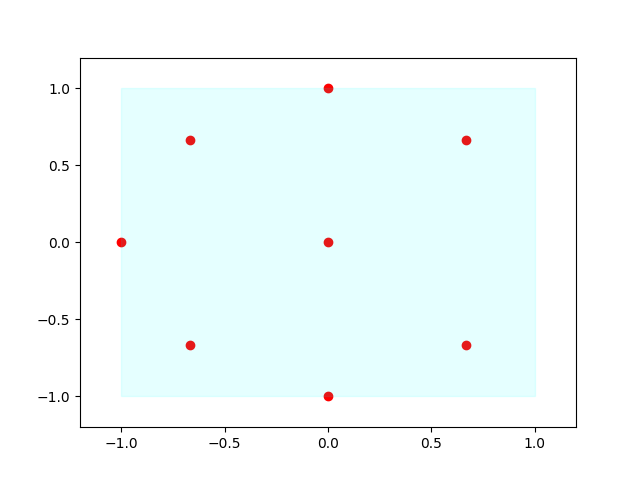}\label{fig:sub2}}\hfill
{\includegraphics[width=0.49\linewidth]{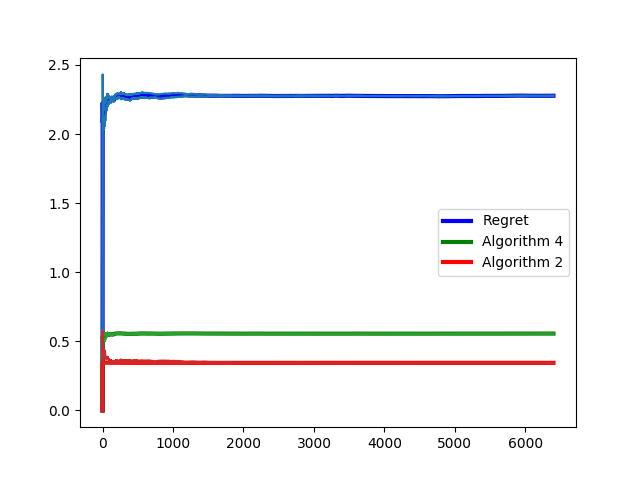}\label{fig:sub3}}\hfill
\caption{The \textcolor{green} {green curve} depicts the time-average connection cost Algorithm~\ref{alg:det}, the \textcolor{red}{red curve} depicts the time-average fractional connection cost of Algorithm~\ref{alg:frac_no_regret} and the \textcolor{blue}{blue curve} depicts their ratio that acts as an upper bound on the regret.
}\label{f:uniform_square}
\end{figure}

\textbf{Uniform Distribution with Time-Evolving Centers} In this case
the $20$ clients arrive according to the uniform distribution with radius $0.3$ and a time-varying center that periodically follows the trajectory described in Example~\ref{ex:1}. Figure~\ref{f:circle2} depicts the centers at which Algorithm~\ref{alg:det} converges after $100k^2$ rounds which are clearly close to the optimal ones.

\textbf{Moving-Clients on the Ellipse}
In this case the $20$ clients move in the ellipse $\left(\frac{x}{1.2}\right)^2 + \left(\frac{y}{0.6}\right)^2=1$ with different speeds and initial positions. The position of client $i$ is given by $\left(x_i(t),y_i(t)\right) = \left(1.2 \cos ( 2\pi f_i t + \theta_i ) , 0.6 \sin \left( 2\pi f_i t + \theta_i \right)\right)$ where
each $f_i,\theta_i$ was selected uniformly at random in $[0,1]$. Figure~\ref{fig:circle} illustrates how Algorithm~\ref{alg:det} converges to the underlying ellipse as the number of rounds increases. 
\begin{figure}[!htb]
\centering
    {\includegraphics[width=0.45\linewidth]{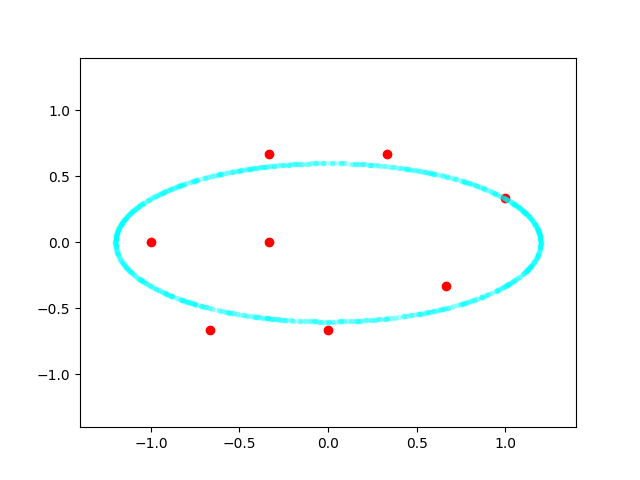}\label{fig:sub1}}\hfill
  {\includegraphics[width=0.45\linewidth]{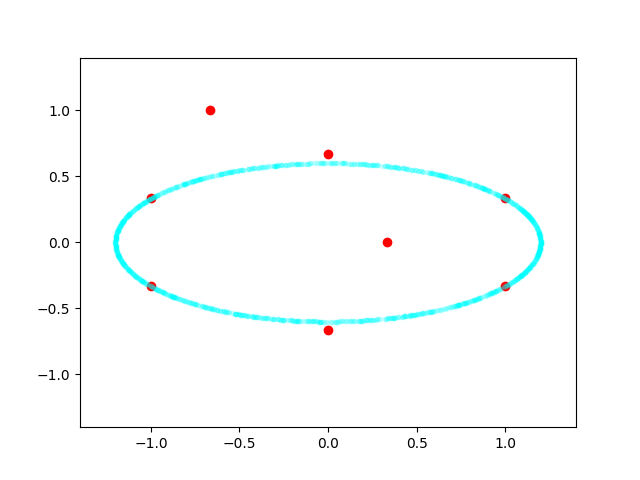}\label{fig:sub1}}\hfill
{\includegraphics[width=0.5\linewidth]{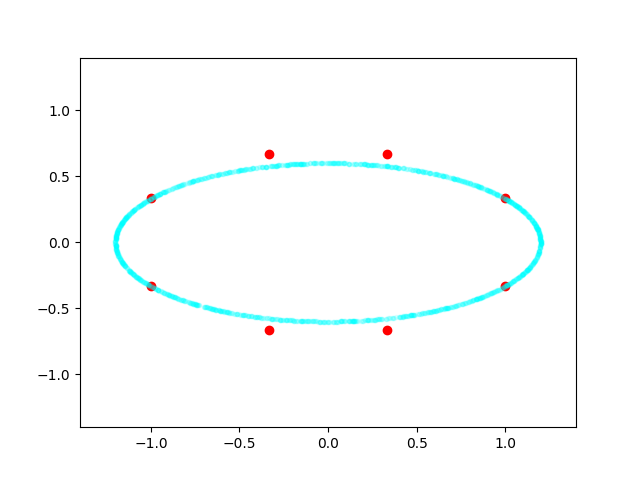}\label{fig:sub1}}\hfill
  \caption{The solution produced by Algorithm~\ref{alg:det} for $k=8$ after $100$, $1000$ and $10000$ rounds.}
    \label{fig:circle}
\end{figure}

\textbf{Mixture of Multivariate Guassians} In this case 15 clients arrive according to the Gaussian with $\mu_1 = (-0.7,0.7)$ and $\Sigma_1
=[[0.3,0],[0,0.3]]$ and $5$ according to the Gaussian with $\mu_2 = (0.7,-0.7)$ and $\Sigma_2
=[[0.3,0],[0,0.3]]$. All the clients outside the $[-1,1]\times [-1,1]$ are projected back to the square. Figure~\ref{f:gaussian} illustrates the 
solutions at which Algorithm~\ref{alg:det} converges for $k=2,8$ and $16$.
\begin{figure}[!htb]
\centering
  {\includegraphics[width=0.49\linewidth]{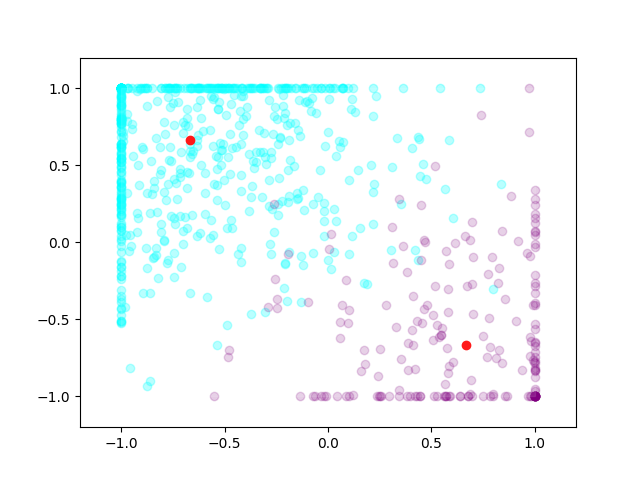}\label{fig:sub2}}\hfill
  {\includegraphics[width=0.49\linewidth]{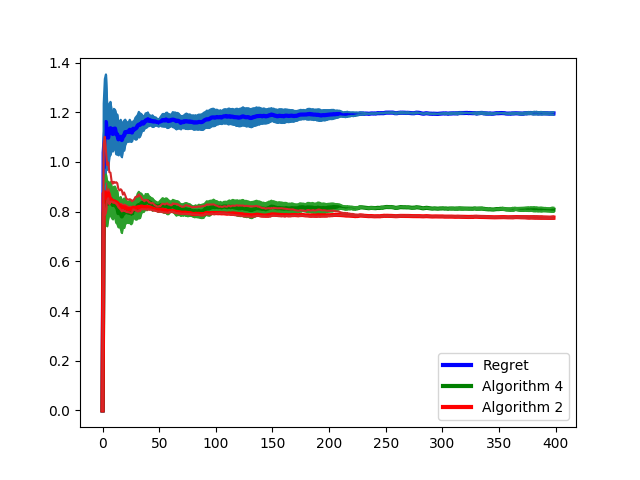}\label{fig:sub3}}\hfill
{\includegraphics[width=0.49\linewidth]{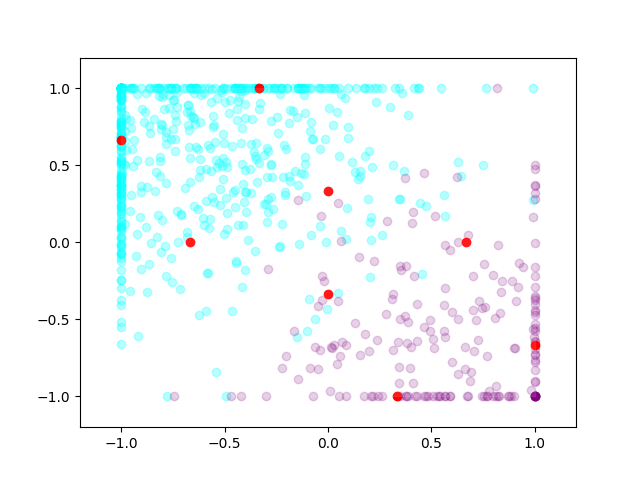}\label{fig:sub2}}\hfill
 {\includegraphics[width=0.49\linewidth]{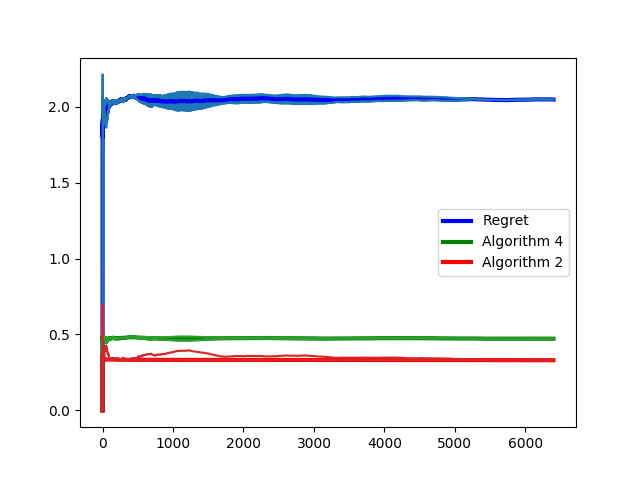}\label{fig:sub3}}\hfill
 {\includegraphics[width=0.49\linewidth]{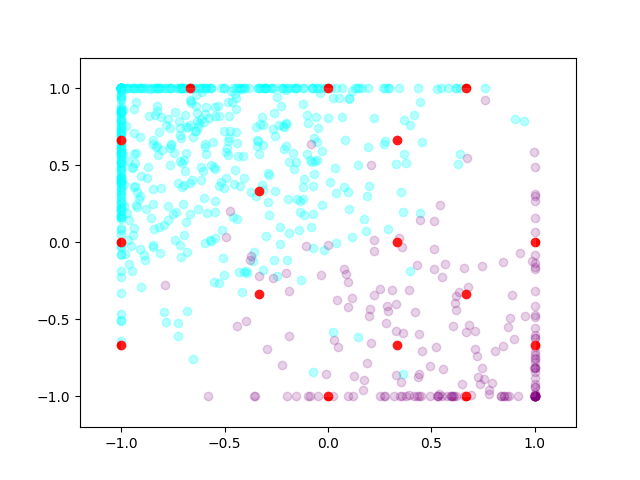}\label{fig:sub2}}\hfill
 {\includegraphics[width=0.49\linewidth]{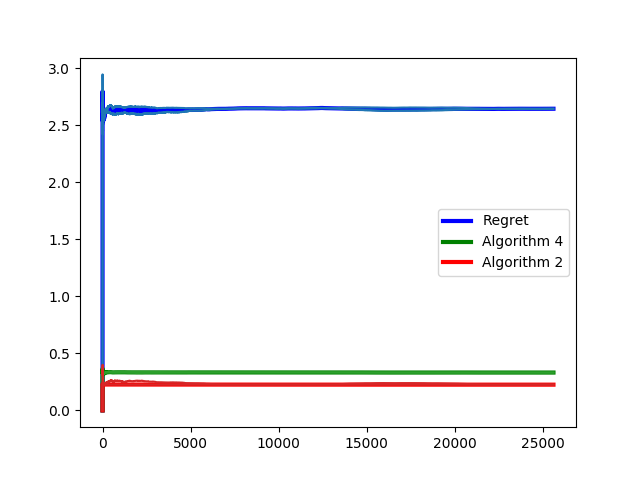}\label{fig:sub3}}\hfill
\caption{On the left, the solutions which Algorithm~\ref{alg:det} converges for $k=2,8$ and $k=16$. On the right, the time-average cost of Algorithm~\ref{alg:det}, Algorithm~\ref{alg:frac_no_regret} and the regret bounds.}
\label{f:gaussian}
\end{figure}
\section{Conclusion}
This work studies polynomial-time low-regret online learning algorithms for Dynamic $k$-Clustering, an online learning problem capturing clustering settings with time-evolving clients for which no information on their locations over time is available. We show that, under some well-established conjectures, $O(1)$-regret cannot be achieved in polynomial time and we provide a $\Theta(\min(k,r))$-regret polynomial time algorithm with $r$ being the maximum number of clients appearing in a single round. At a technical level, we present a two-step approach where in the first step we provide a no-regret algorithm for the Fractional Dynamic $k$-Clustering while in the second step we provide online rounding scheme converting the sequence of fractional solutions, produced by the no-regret algorithm, into solutions of Dynamic $k$-Clustering. Applying the same approach to other combinatorial online learning problems is an interesting research direction.

\bibliographystyle{plain}
\bibliography{example_paper}

\newpage
\onecolumn
\appendix
\section*{Appendix}

\section{Proof of Theorem~\ref{t:hardnes}}
\begin{problem}[$k-\mathrm{CenterUniform}$]
Given a uniform metric space $d:V \times V \mapsto \R_{\geq 0}$
($d(u,v) = 1$ in case $(u \neq v)$) and a set of requests
$R_1,\ldots,R_m \subseteq V$. Select $F \subseteq V$ such as $|F|=k$ and $\sum_{s=1}^m C_{R_s}(F)$ is minimized where $p$ is $\infty$.
\end{problem}

\begin{lemma}
Any $c$-approximation algorithm for $k-\mathrm{CenterUniform}$ implies a $c$-approximation algorithm for $\mathrm{Min}-p-\mathrm{Union}$.
\end{lemma}
\begin{proof}
Given the collection $U =\{S_1, \ldots, S_m\}$ of the $\mathrm{Min}-p-\mathrm{Union}$, we construct a uniform metric space $V$ of size $m$, where each node of $V$ corresponds to a set $S_i$.

For each elements $e \in E$ of the $\mathrm{Min}-p-\mathrm{Union}$ we construct a request $R_e \subseteq V$ for the $k-\mathrm{CenterUniform}$ that is composed by the nodes corresponding to the sets $S_i$ that containt $e$. Observe that due to the uniform metric and the fact that $p = \infty$, for any $V' \subseteq V$
\[\sum_{e \in E}C_{R_e}(V') = |\cup_{S_i \notin V'}S_i|\]
\end{proof}

\begin{lemma}
Any polynomial time $c$-regret algorithm for the online $k$-Center implies a $(c+1)$-approximation algorithm (offline) for the $k-\mathrm{CenterUniform}$.
\end{lemma}
\begin{proof}
Let assume that that there exists a polynomial-time online learning algorithm such that for any request sequence $R_1,\ldots,R_T$,
\[\sum_{t=1}^T \E[C_{R_t}(F_t)] \leq c~ \min_{|F^\ast|=k} \sum_{t=1}^T \E[C_{R_t}(F^\ast)] + \Theta(\mathrm{poly}(n,D) \cdot T^\alpha)\]
for some $\alpha < 1$.

Now let the requests lie on the uniform metric, 
$p=\infty$ and that the adversary at each round $t$ selects uniformly at random one of the requests $R_1,\ldots,R_m$ that are given by the instance of $k-\mathrm{CenterUniform}$. In this case the above equation takes the following form,
\[\sum_{t=1}^T \frac{1}{m} \sum_{s=1}^m \E[C_{R_s}(F_t)] \leq c~\frac{T}{m}\sum_{s=1}^m \E[C_{R_s}(\mathrm{OPT}^\ast)] + \Theta(n^\beta \cdot T^\alpha)\]
where $\mathrm{OPT}^\ast$ is the optimal solution for the instance of $k-\mathrm{CenterUniform}$ and $F_t$ is the random set that the online algorithm selects at round $t$.

Now consider the following randomized algorithm for the 
$k-\mathrm{CenterUniform}$.
\begin{enumerate}

    \item Select uniformly at random a $t$ from $\{1,\ldots,T\}$.
    \item Select a set $F \subseteq V$ according to the probability distribution $F_t$.
\end{enumerate}
The expected cost of the above algorithm, denoted by $\E[\mathrm{\mathrm{ALG}}]$, is
\begin{eqnarray*}
\frac{1}{T}\sum_{t=1}^T \sum_{i=1}^m \E_{F \sim F_t}\left[C_{R_i}(F)\right]
&=& m \cdot \left( \frac{1}{T} \sum_{t=1}^T \sum_{i=1}^m
\frac{1}{m}\E_{F\sim F_t}\left[C_{R_i}(F)\right] \right)\\
&\leq& \frac{c \cdot m}{T} \cdot \frac{T}{m} \sum_{i=1}^m C_{R_i}(\mathrm{OPT}^\ast)\\
&+& \Theta \left( \frac{ m \cdot n^{\beta}}{T^{1 - \alpha}}\right)
\end{eqnarray*}
By selecting $T = \Theta(m^{\frac{1}{1-\alpha}} \cdot n^{\frac{\beta}{1-\alpha}})$ we get that $\E[\mathrm{ALG}] \leq (c+1)\cdot \mathrm{OPT}^\ast$.
\end{proof}

\section{Proof of Theorem~\ref{t:lower_bound_det}}
Let the metric space be composed by $k+1$ points with the distance between any pair of (different) points being $1$. At each round $t$, there exists a position at which the learner has not placed a facility (there are $k+1$ positions and $k$ facilities). If the adversary places one client at the empty position of the metric space, then the deterministic online learning algorithm admits overall connection cost equal to $T$. However the optimal static solution that leaves empty the position with the least requests pays at most $T/(k+1)$. 

\section{Omitted Proof of Section~\ref{s:fractional}}\label{app:fractional}

\subsection{Proof of Lemma~\ref{l:dual}}
The Langragian of the convex program of Definition~\ref{d:frac_cost} is,
\begin{eqnarray*}
L(\beta, y, x,A,k,\lambda) &=& \left(\sum_{j \in R} \beta_j^p\right)^{1/p}\\
&+& \sum_{j \in R}\lambda_j \cdot \left(d_{ij}x_{ij} - \beta_j\right) +
\sum_{j \in R}A_j \cdot \left(1 - \sum_{i\in V}x_{ij} \right)\\
&+& \sum_{i \in V}\sum_{j \in V} k_{ij}\cdot ( x_{ij} - y_i) - \sum_{i \in V}\sum_{j \in R} \mu_{ij}\cdot x_{ij}
\end{eqnarray*}
Rearranging the terms we get,

\begin{eqnarray*}
L(\beta, y, x,A,k,\lambda) &=& \sum_{j \in R} A_j - \sum_{i \in V}\sum_{j \in R}k_{ij}\cdot y_i\\
&+& \sum_{i \in V}\sum_{j \in R} x_{ij} \cdot \left(k_{ij}
- \mu_{ij} + d_{ij} \cdot \lambda_j - A_j\right)\\
&+& \left(\sum_{j \in R} \beta_j^p\right)^{1/p} - \sum_{j \in R}\lambda_j\cdot \beta_j
\end{eqnarray*}
In order for the function $g(A,k,\lambda) = \min_{\beta,y,x,M^+,M^-} L(\beta, y, x,A,k,\lambda)$ to get a finite value the following constraints must be satisfied,

\begin{itemize}
    \item $k_{ij} + d_{ij} \cdot \lambda_j - A_j = \mu_{ij}$
    \item $||\lambda||_{p}^\ast \leq 1$ since otherwise $\left(\sum_{j \in R} \beta_j^p\right)^{1/p} - \sum_{j \in R}\lambda_j\cdot \beta_j$ can become $-\infty$.
\end{itemize}
Using the fact that the Lagragian multipliers $\mu_{ij}\geq 0$, we get the constraints of the convex program of Lemma~\ref{l:dual}. The objective comes from the fact that once  
$g(A,k,\lambda)$ admits a finite value then 
$g(A,k,\lambda) = \sum_{j \in R}A_j - \sum_{i \in V}\sum_{j \in R} k_{ij} \cdot y_i$.

\subsection{Proof of Lemma~\ref{l:subgradients}}
Let $\lambda_j^\ast,A_j^\ast,k_{ij}^\ast$ denote the values of the respective variables in the optimal solution of the convex program of Lemma~\ref{l:dual} formulated with respect to the vector $y= (y_1,\ldots,y_n)$. Respectively consider $\lambda_j',A_j',k_{ij}'$ denote the values of the respective variables in the optimal solutions of the convex program of Lemma~\ref{l:dual} formulated with respect to the vector $y'= (y_1',\ldots,y_n')$.

\begin{eqnarray}
\mathrm{FC}_{R}(y') &=& \sum_{j \in R}A_j' - \sum_{i \in V }\sum_{j \in R}k_{ij}'\cdot y_i'\\
&\geq& \sum_{j \in R}A_j^\ast - \sum_{i \in V}\sum_{j \in R}k_{ij}^\ast \cdot y_i'\\
&=& \sum_{j \in R}A_j^\ast - \sum_{i \in V}\sum_{j \in R}k_{ij}^\ast \cdot y_i'+ \sum_{i \in V}\sum_{j \in R}k_{ij}^\ast \cdot y_i - \sum_{i \in V}\sum_{j \in R}k_{ij}^\ast \cdot y_i\\
&=& \mathrm{FC}_{R}(y) + \sum_{i \in V}\sum_{j \in R}k_{ij}^\ast \cdot (y_i - y_i')
\end{eqnarray}
Equations~$5$ and~$6$ follow by strong duality, more precisely $\mathrm{FC}_{R}(y) = \sum_{j \in R}A_j^{\ast} - \sum_{i \in V }\sum_{j \in R}k_{ij}^\ast \cdot y_i$ since the convex program of Lemma~\ref{l:dual} is the dual of the convex program the solution of which defines $\mathrm{FC}_{R}(y')$ (respectively for $\mathrm{FC}_{R}(y') = \sum_{j \in R}A_j' - \sum_{i \in V }\sum_{j \in R}k_{ij}' \cdot y_i'$). Equation $4$ is implied by the fact that the solution $(\lambda',k',A')$ is optimal when the objective function is $\sum_{j \in R}A_j - \sum_{i \in V}\sum_{j \in R} k_{ij}\cdot y_i'$. Notice that the constraints of the convex program in Lemma~\ref{l:dual} do not depend on the $y$-values. As a result, the solution $(\lambda^\ast,k^\ast,A^\ast)$ (that is optimal for the dual convex program formulated for $y$) is feasible for the dual program formulated for the values $y'$. Thus Equation~$4$ follows by the optimality of $(\lambda',k',A')$.

Up next we prove the correctness of Algorithm~\ref{alg:dual}. Notice that the the solution $\beta,x$ that Algorithm~\ref{alg:dual} constructs is feasible for the primal convex program of Definition~\ref{d:frac_cost}. We will prove that the dual solution that Algorithm~\ref{alg:dual} constructs is feasible for the dual of Lemma~\ref{l:dual}
while the exact same value is obtained. 

\begin{itemize}
    \item \underline{$|| \lambda||_{p}^\ast = 1$:} It directly follows by the fact that $\lambda_j = \left[ \frac{\beta_j}{||\beta||_p} \right]^{p-1}$ and  
    $||\lambda||_{p}^\ast = \left[\sum_{j \in R} \lambda_j^{\frac{p}{p-1}}\right]^{\frac{p-1}{p}}$.
    
    \item $\underline{d_{ij}\cdot \lambda_j + k_{ij} \geq A_j:}$ In case $d_{ij} < D_j^\ast$, Algorithm~\ref{alg:dual} implies that $x_{ij} = y_i$ and the inequality directly follows. In case $d_{ij} \leq D_j^\ast$ the inequality holds trivially since $k_{ij} = 0$.  
\end{itemize}
Now consider the objective function,

\begin{eqnarray}
\sum_{j \in R} A_j - \sum_{i \in V}\sum_{j \in R} y_i \cdot k_{ij} &=&
\sum_{j \in R} A_j - \sum_{j \in R} \sum_{i \in V_j^+} y_i \cdot k_{ij} \nonumber\\
&=&
\sum_{j \in R} \lambda_j\cdot D_j - \sum_{j \in R} \sum_{i \in V_j^+} y_i \left[ \lambda_j \cdot \frac{x_{ij}}{y_i} \left( D_j - d_{ij} \right) \right] \nonumber\\
&=& \sum_{j \in R} \lambda_j \sum_{i \in V_j^+} d_{ij}\cdot x_{ij} \\
&=& \sum_{j \in R}\lambda_j \cdot \beta_j \nonumber\\
&=& \left( \sum_{j \in R} \beta_j^p \right)^{1/p} \nonumber
\end{eqnarray}
where Equation~$8$ follows by the fact that $x_{ij} = 0$ for all $j \notin V_{j}^+$ and thus $\sum_{j \in V_j^+} x_{ij} = 1$. Finally notice that $|\lambda_j| \leq 1 $ and thus $k_{ij} \leq D$ where $D$ is the diameter of the metric space.

\subsection{Proof of Theorem~\ref{t:no-regret-frac}}
By Lemma~\ref{l:subgradients}, $|g_i^t| = |-\sum_{j \in R^t}k_{ij}^{t\ast}| \leq Dr$ since $|R^t| \leq r$. Applying Theorem~$1.5$ of \cite{H16} we get that
\[ \sum_{t=1}^T \sum_{i \in V} g_i^t (y_i^t - y^\ast_i) \leq \Theta \left( kD r \sqrt{\log n T} \right)\]
Applying Lemma~\ref{l:subgradients} for $y' = y^\ast$,
\[ \sum_{t=1}^T \left(\mathrm{FC}_{R_t}(y^t) - \mathrm{FC}_{R_t}(y^\ast)\right) \leq  \sum_{t=1}^T \sum_{i \in V} g_i^t (y_i^t - y^\ast_i) \leq \Theta \left( kD r \sqrt{\log n T} \right)\]

\section{Omitted Proof of Section~\ref{s:det}}\label{app:deterministic}

\subsection{Proof of Lemma~\ref{l:rounding_lemma}}

The following claim trivially follows by Step~10 of Algorithm~\ref{alg:det}.
\begin{claim}\label{c:1}
For any node $j \in V$, $d(j,F_y) \leq 6k \cdot \beta_j^\ast$.
\end{claim}

We are now ready to prove the first item of Lemma~\ref{l:rounding_lemma}. Let a request $R \subseteq V$,

\[\mathrm{C}_{R}(F_y) = \left( \sum_{j \in R} d(j,F_y)^p \right)^{1/p}
\leq \left( \sum_{j \in R} (6k)^p \cdot \beta_j^{\ast p} \right)^{1/p}
= 6k \cdot \left( \sum_{j \in R} \beta_j^{\ast p} \right)^{1/p}
\]

We proceed with the second item of Lemma~\ref{l:rounding_lemma}. For a given node $j \in S$, let $B_j =\{i \in V: d_{ij} \leq 3k \cdot \beta_j^\ast\}$. It is not hard to see that for any $j \in F_y$,
\[\sum_{i \in B_j}y_i \geq 1 - \frac{1}{3k}\]
Observe that in case the latter is not true then $\sum_{i \notin B_j}x_{ij}^\ast \geq \frac{1}{3k}$, which would imply that $\beta_j^\ast > \beta_j^\ast$. 

The second important step of the proof is that for any $j,j' \in F_y$, \[B_j \cap B_{j'} = \emptyset.\]
\noindent Observe that in case there was $m \in B_j \cap B_{j'}$ would imply $d(j,m) \leq 3k \cdot \beta_j^\ast$ and
$d(j',m) \leq 3k \cdot \beta_{j'}^\ast$. By the triangle inequality we get $d(j,j') \leq 6k \cdot \beta_{j'}^\ast$ (without loss of generality $\beta_j^\ast \leq \beta_{j'}^\ast$). The latter contradicts with the fact that both $j$ and $j'$ belong in set $F_y$.

Now assume that $|F_y| \geq k + 1$. Then $\sum_{i \in F_y} y_i \geq |F_y|\cdot (1 - \frac{1}{3k}) \geq (k+1)\cdot (1 - \frac{1}{3k}) > k$. But the latter contradicts with the fact that $\sum_{i \in V}y_i = k$. As a result, $|F_y| \leq k$.

\section{Omitted Proofs of Section~\ref{s:rand}}\label{app:rand}

\begin{proof}[Proof of Theorem~\ref{t:rand-regret}]
To simplify notation the quantity $\E_{F \sim \mathrm{CL}(y_t)}[C_{R_t}(F)]$ is denoted as $\E[C_{R_t}(F_t)]$. At first notice that by the first case of Lemma~\ref{l:Charikar-Lin}, Algorithm~\ref{alg:rand} ensures that exactly $k$ facilities are opened at each round $t$.

Concerning its overall expected connection cost we get,
\[
\E\left[C_{R_t}(F_{y_t})\right] \leq \sum_{j \in R_t}\E[C_{\{j\}}(F_{y_t})]
\leq 4 \sum_{j \in R_t}\mathrm{FC}_{\{j\}}(y_t)
\]
where the fist inequality is due to the fact that $\sum_{j \in R_t}\mathrm{d}(j,F)^p \leq \left(\sum_{j \in R_t}\mathrm{d}(j,F)\right)^p$ and the second is derived by applying the second case of Lemma~\ref{l:Charikar-Lin}. We overall get,
\begin{eqnarray}
\sum_{t=1}^T \E[C_{R_t}(F_{y_t})]
&\leq& 4 \sum_{t=1}^T\sum_{j \in R_t}\mathrm{FC}_{\{j\}}(y_t) \nonumber\\
&\leq& 4\sum_{t=1}^T |R_t| \cdot \mathrm{FC}_{R_t}(y_t)\\
&\leq& 4r \min_{y^\ast} \sum_{t=1}^T
\mathrm{FC}_{R_t}(y^\ast) \nonumber\\ &+& \Theta \left( kD r \sqrt{\log n T} \right) \nonumber\\
&\leq& 4r \min_{|F^\ast| = k} \sum_{t=1}^T
\E[ \mathrm{C}_{R_t}(F^\ast)] \nonumber\\ &+& \Theta \left( kD r \sqrt{\log n T} \right)\nonumber
\end{eqnarray}
where inequality~$3$ follows by the fact that $\mathrm{FC}_{\{j\}}(y) \leq \mathrm{FC}_{\{R\}}(y)$ for all $j \in R$
and the last two inequalities follow by Theorem~\ref{t:no-regret-frac} and Lemma~\ref{l:frac_int} respectively.
\end{proof}

\end{document}